\documentclass[final,1p,times]{elsarticle}

\usepackage{amssymb}

\usepackage{amsmath}
\usepackage{stmaryrd}
\usepackage{xcolor}
\usepackage{xspace}
\usepackage{caption}
\usepackage{hyperref}

\usepackage[T1]{fontenc}
\usepackage{txfonts}

\usepackage{interval}

\usepackage{textgreek}

\usepackage{amsthm}

\usepackage{xspace}
\usepackage{amsfonts}
\usepackage{bbm}
\usepackage{textcomp}

\newcommand{\wmc}{\ensuremath{\mathrm{WMC}}\xspace}
\newcommand{\wmi}{\ensuremath{\mathrm{WMI}}\xspace}
\newcommand{\lwmc}{\ensuremath{\mathrm{L{-}WMC}}\xspace}
\newcommand{\lwmi}{\ensuremath{\mathrm{L{-}WMI}}\xspace}

\newcommand{\ite}{\ensuremath{\mathrm{ite}}\xspace}

\newcommand{\weight}{\ensuremath{w}\xspace}

\newcommand{\ive}[1]{\left\llbracket#1\right\rrbracket}
\newcommand{\varset}[1]{\ensuremath{\mathbf{#1}}}

\newcommand{\xvars}{\varset{x}}
\newcommand{\bvars}{\varset{b}}

\newcommand{\lra}{\ensuremath{\mathcal{LRA}}\xspace}
\newcommand{\nra}{\ensuremath{\mathcal{NRA}}\xspace}
\newcommand{\ra}{\ensuremath{\mathcal{RA}}\xspace}

\newcommand{\bdomain}{\ensuremath{\mathbb{B}}\xspace}
\newcommand{\rdomain}{\ensuremath{\mathbb{R}}\xspace}
\newcommand{\zdomain}{\ensuremath{\mathbb{Z}}\xspace}
\newcommand{\borel}{\ensuremath{\mathcal{B}}\xspace}
\newcommand{\bmeas}{\ensuremath{\mu}\xspace}
\newcommand{\rmeas}{\ensuremath{\lambda}\xspace}
\newcommand{\bprob}{\ensuremath{\eta}\xspace}
\newcommand{\rprob}{\ensuremath{\tau}\xspace}
\newcommand{\zmeas}{\ensuremath{\xi}\xspace}
\newcommand{\partitive}{\ensuremath{\mathcal{P}}\xspace}

\newcommand{\false}{\ensuremath{\bot}\xspace}
\newcommand{\true}{\ensuremath{\top}\xspace}
\newcommand{\litb}{\mathcal{L}_{\bvars}\xspace}
\newcommand{\modls}[1]{\mathcal{M}(#1)}
\newcommand{\indicator}[1]{\mathbbm{1}_{#1}}

\newcommand{\tuple}[1]{\left(#1\right)}
\newcommand{\set}[1]{\left\{\,#1\,\right\}}

\theoremstyle{plain}
\newtheorem{theorem}{Theorem}
\newtheorem{proposition}[theorem]{Proposition}
\newtheorem{corollary}[theorem]{Corollary}
\newtheorem{lemma}[theorem]{Lemma}
\theoremstyle{definition}
\newtheorem{definition}[theorem]{Definition}
\newtheorem{example}[theorem]{Example}
\theoremstyle{remark}

\newenvironment{talign}
 {\align}
 {\endalign}

\makeatletter
\def\ps@pprintTitle{%
 \let\@oddhead\@empty
 \let\@evenhead\@empty
 \def\@oddfoot{\centerline{\thepage}}%
 \let\@evenfoot\@oddfoot}
\makeatother

\begin{document}

\begin{frontmatter}


\title{Measure Theoretic Weighted Model Integration\tnoteref{titlefn}}
\tnotetext[titlefn]{Ivan Mio\v{s}i\'{c} conducted this research partially at the KU Leuven while being supported by the Erasmus+ program of the European Union. Pedro Zuidberg Dos Martires is supported by the Special Research Fund of the KU Leuven.}


 \author[1]{Ivan Mio\v{s}i\'{c}}
 \address[1]{University of Zagreb}
 \author[2]{Pedro Zuidberg Dos Martires}
 \address[2]{KU Leuven}







\begin{abstract}
Weighted model counting (WMC) is a popular framework to perform probabilistic inference with discrete random variables. 
Recently, WMC has been extended to weighted model integration (WMI) in order to additionally handle continuous variables.
At their core, WMI problems consist of computing integrals and sums over weighted logical formulas.  
From a theoretical standpoint, WMI has been formulated by patching the sum over weighted formulas, which is already present in WMC, with Riemann integration.
A more principled approach to integration, which is rooted in measure theory, is Lebesgue integration.
Lebesgue integration allows one to treat discrete and continuous variables on equal footing in a principled fashion.
We propose a theoretically sound measure theoretic formulation of weighted model integration, which naturally reduces to weighted model counting in the absence of continuous variables.
Instead of regarding weighted model integration as an extension of weighted model counting, WMC emerges as a special case of WMI in our formulation.
\end{abstract}


\end{frontmatter}


\section{Introduction}
Weighted model counting (WMC)~\cite{chavira2008probabilistic}, in combination with knowledge compilation~\cite{darwiche2002knowledge}, has emerged as the go-to technique to perform inference in probabilistic graphical models~\cite{darwiche2009modeling} and probabilistic programming languages~\cite{fierens2015inference} with discrete random variables. 
A  major drawback of standard WMC, however, is its limitation to discrete (random) variables and hence to discrete probability distributions and weight functions only.
This puts considerable restrictions on the problems that can be modeled.
Weighted model integration (WMI)~\cite{belle2015probabilistic} is a recent extension of the WMC formalism that tackles this deficiency and allows additionally for continuous variables.

\begin{example}
Consider the example of a WMI problem in Figure~\ref{figure:intro_example_wmi}.
The problem has two continuous random variables ($x$ and $y$) and three Boolean random variables, which produce the different feasible regions (the red region and the two blue regions).
The regions themselves are given by constraints on the continuous variables. Moreover, for each feasible region a weight function is given. Outside of the regions the weight is zero.
WMI tackles the problem of computing the integral over the feasible regions.
\end{example}

\begin{figure}[h]
    \centering
    \includegraphics[width=0.65\linewidth]{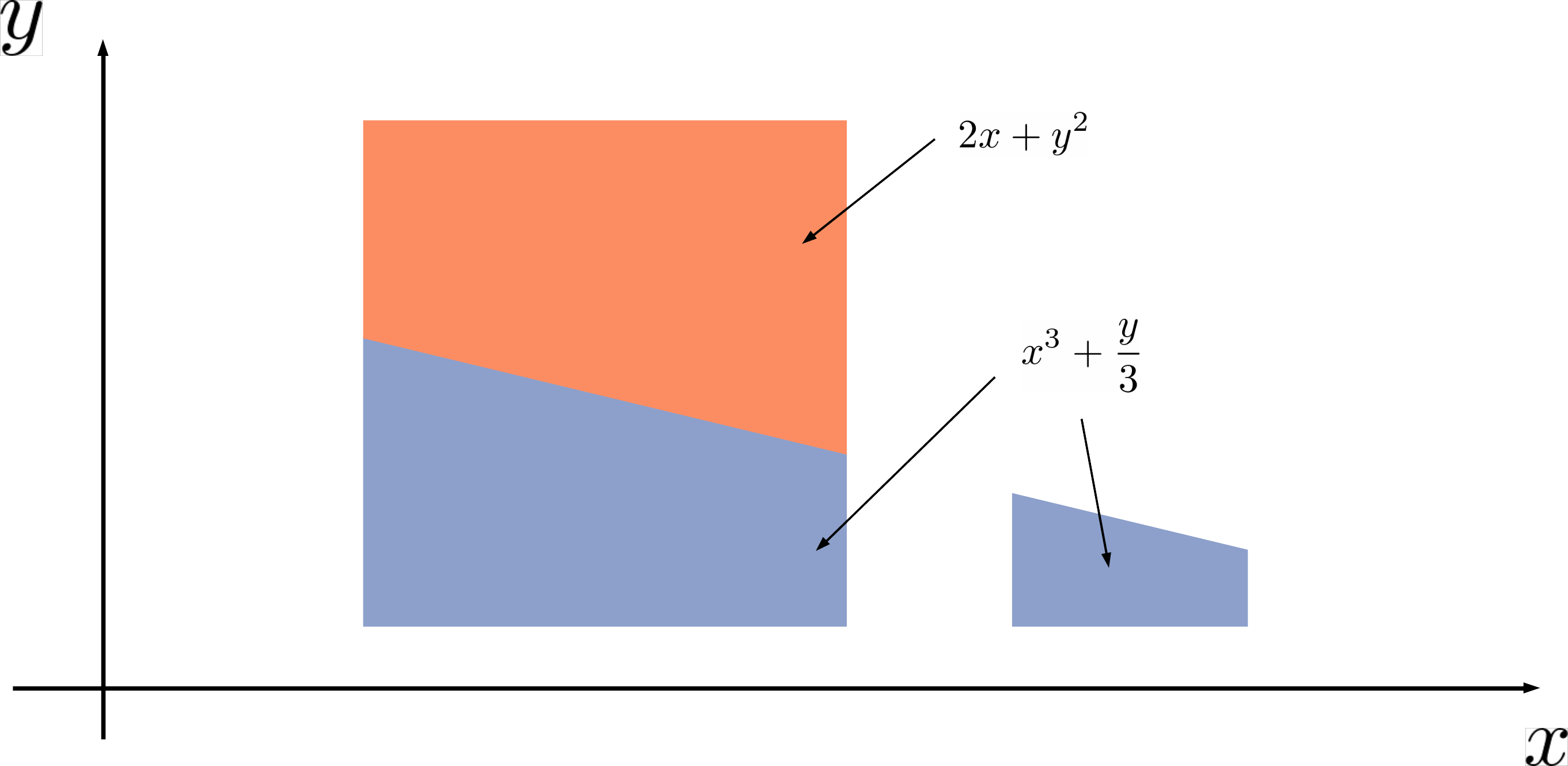} 
    \caption{Geometric representation of a WMI problem.}
    \label{figure:intro_example_wmi}
\end{figure}

Since the inception of WMI, a plethora of inference algorithms have emerged following the WMI paradigm.
Some of which perform exact inference~\cite{braz2016probabilistic,belle2016component,morettin2017efficient,merrell2017weighted,kolb2018efficient,morettin2019advanced,zuidberg2019exact,kolb2019structure,zeng2020scaling,derkinderen2020ordering}, or approximate inference~\cite{belle2015hashing,zuidberg2019exact}, or solve a subclass of WMI problems efficiently~\cite{zengefficient} --- demonstrating an avid interest from the research community.
The pywmi toolbox for WMI solvers~\cite{kolb2019pywmi} reassembles some of these efforts in a single Python library.

All of the above cited works formalize WMI as a combination of Riemann integration  and summation and none leverages the power of Lebesgue integration  in order to formally define weighted model integration.
This is rather astounding as Lebesgue integration is a natural fit to formalize integration/summation in such a discrete-continuous setting.
We speculate this is due to the technical overhead involved with Lebesgue integration compared to Riemann integration.
In this paper we show how the problem of weighted model integration can be defined in terms of Lebesgue integration and place WMI in a measure theoretic setting.
We hope this work will help bridge theoretical distinctions between different approaches to WMI and create a unified and cohesive view on probabilistic inference problems in Boolean, discrete, and continuous domains.

Traditionally, probability theory has been one of the main domains of application of Lebesgue integration and measure theory.
Probability is in fact naturally represented as special type of measure function.
Given that WMI is inseparably connected to probabilistic inference, it is convenient to formalize WMI in a measure theoretic setting.
Effectively, this extends the currently used class of Riemann integrable WMI problems~\citep{belle2015probabilistic} to the class of Lebesgue integrable problems. 

The remainder of the paper is organized as follows.
Section~\ref{sec:measure_theory} discusses the necessary background on measure theory.
In Section~\ref{sec:wmc_wmi}, the problems of WMC and WMI are introduced 
(using the current formulation based on summation and Riemann integration).
Section~\ref{sec:mt_wmc_wmi}, the central part of the paper, first presents the formalization of WMC in a measure theoretic setting (Subsection~\ref{sec:mt_wmc}),
followed by an analogous treatment of WMI (Subsection~\ref{sec:mt_wmi}).
Succeedingly, Section~\ref{sec:measures} deals with the important category of probabilistic weight functions, which can directly be represented as a measure, again first in the Boolean setting (Subsection~\ref{sec:wmc_measure}), and then in the hybrid one (Subsection~\ref{sec:wmi_measure}).
We end the paper with concluding remarks in Section~\ref{sec:conclusion}.

\section{WMC and WMI}\label{sec:wmc_wmi}

\subsection{Weighted Model Counting}\label{sec:wmc}

\begin{definition}
Let $\bvars = \set{B_1, B_2, \ldots, B_M}$ be a set of $M$ Boolean variables (or logical propositions), which can be combined in the usual way using logical connectives $\lnot$, $\land$ and $\lor$, producing \emph{formulas of propositional logic}.
We call a \emph{literal} either a Boolean variable or its negation, and denote with $\litb = \set{B_1, \lnot B_1, B_2, \lnot B_2, \dotsc, B_M, \lnot B_M}$ the set of all literals over the set $\bvars$.
\end{definition}

Without loss of generality, we regard the sets of variables as ordered sets. Also, any logical formula we mention is assumed to use all variables from the underlying set of variables.
The set of Boolean (truth) values will be denoted by $\bdomain = \set{\false, \true}$.
In order to assign a truth value to a formula, we introduce the concept of an \emph{interpretation}.

\begin{definition}[Interpretation of propositional formula]\label{def:inter_bool}
A \emph{total interpretation} of the Boolean variables in $\bvars$ is any mapping from set $\bvars$ to the set $\bdomain$.
We require this mapping to commute with logical connectives in the usual way, so that it can be extended to any propositional formula built from variables in $\bvars$.
A propositional formula $\phi$ is said to be \emph{true} under the interpretation $I$ if $I(\phi) = \true$ and \emph{false} otherwise.
\end{definition}

Closely connected to the notion of interpretation is that of a \emph{model}. 
We define it conveniently for later use in our formulation of measure theoretic WMC and WMI.

\begin{definition}[Model of propositional formula]\label{def:model_bool}
Let $I$ be an interpretation and $\phi$ a propositional formula over $\bvars$, such that $I(\phi) = \true$.
We say that the $M$-tuple
\[ \mathfrak{M}(I) = \tuple{I(B_1), I(B_2), \ldots, I(B_M)} \in \bdomain^M \]
is a \emph{model} of $\phi$ associated with interpretation $I$. We denote the \emph{set of all models} of the propositional formula $\phi$ by $\modls{\phi} = \set{\mathfrak{M}(I) \mid I(\phi) = \true}$.
\end{definition}

In the WMC literature, the model of a propositional formula associated with an interpretation $I$ is traditionally defined as a subset of $\litb$ containing literals that are true under this interpretation~\cite{chavira2008probabilistic,kimmig2017algebraic}.
Any such subset $\mathfrak{M}^{\mathcal{L}}(I) = \set{\ell_1, \ell_2, \ldots, \ell_M}$, where $\ell_i = \ite \left( I(B_i), B_i, \lnot B_i \right)$ for all $i = 1$, $2$, $\dotsc$, $M$, uniquely defines the $M$-tuple $\tuple{I(B_1), I(B_2), \ldots, I(B_M)} \in \bdomain^M$ used in the previous Definition, and vice versa.
The function symbol `$\ite$' denotes the {\em if-then-else} function: if the first argument is $\true$ (true) the second argument is returned, else the third argument is returned.


Using this notation, the well known Boolean satisfiability problem (SAT) is expressed as the problem of determining whether $\modls{\phi} = \emptyset$. Its counting counterpart (\#SAT) is expressed as determining the exact number of elements in $\modls{\phi}$.

\begin{definition}[WMC]\label{def:wmc}
Let $\bvars$ be a set of $M$ Boolean variables, and $\phi$ be a propositional formula over $\bvars$. Furthermore, let $\weight^{\mathcal{L}} \colon \litb \to \mathbb{R}_{\ge 0}$ be a weight function of Boolean literals. Then the \emph{weighted model count (WMC)} of the formula $\phi$ is given by:
\begin{align}\label{eq:wmc}
    \wmc\left(\phi, \weight^{\mathcal{L}} \mid \bvars \right) = \sum_{\mathfrak{M} \in \modls{\phi}} \prod_{\;\ell \in\, \mathfrak{M}^{\mathcal{L}}} \weight^{\mathcal{L}}(\ell) \text.
\end{align}
\end{definition}


For simplicity of exposition, we assume the weight function to be non-negative, which is also justified by weight functions used in practice.
The importance of WMC for probabilistic inference cannot be overstated, and is thoroughly investigated in~\cite{chavira2008probabilistic,fierens2015inference}. 
Further interesting generalizations of WMC to semirings other than the  $\mathbb{R}$-semiring are discussed in~\cite{kimmig2017algebraic}.

\subsection{Weighted Model Integration}\label{sec:wmi}

Many applications require probabilistic inference in continuous domains.
In order to capture these applications, the task of weighted model counting has been extended to weighted model integration~\cite{belle2015probabilistic}.
The first step is a definition of a logical theory which combines Boolean and continuous variables. To this end we follow the definition in~\cite{zuidberg2019exact} (a more formal definition can be found in~\cite{biere2009handbook}).

\begin{definition}[SMT]
Let $\bvars = \set{B_1, B_2, \ldots, B_M}$ be a set of $M$ Boolean variables, and $\xvars = \set{X_1, X_2, \ldots, X_N}$ be a set of $N$ real variables. 
An \emph{atomic formula} is either a Boolean variable (\textbf{logical proposition}) from set $\bvars$, or a valid arithmetical statement (\textbf{real arithmetical proposition}) consisting of variables from $\xvars$, real numbers and symbols $+$, $\cdot$, \textasciicircum, and $\le$, having standard interpretation as real addition, multiplication, exponentiation, and less-than inequality, respectively. Atomic formulas are combined using logical connectives $\lnot$, $\land$ and $\lor$, producing so-called \emph{SMT formulas}.
\end{definition}

Any real arithmetical proposition $\theta$ can be written in the equivalent general form \linebreak $\hat{\theta}(X_1, X_2, \ldots, X_N) \le 0$. Here $\hat{\theta}$ denotes a function from $\rdomain^N$ to $\mathbb{R}$ encoded by proposition $\theta$. Based on the restrictions posed on this function in a specific SMT theory, we distinguish, among others, SMT(\lra) theory ($\hat{\theta}$ is a linear function), SMT(\nra) theory ($\hat{\theta}$ is a polynomial function) and SMT(\ra) theory ($\hat{\theta}$ is unrestricted).

\begin{definition}[Interpretation of SMT formulas]\label{def:inter_smt_total}
Let an SMT theory be built over the Boolean variables in $\bvars$ and continuous variables in $\xvars$. A \emph{total interpretation} of the variables in $\bvars$ and $\xvars$ is a pair $I = (I_{\bvars}, I_{\xvars})$, where $I_{\bvars}$ is a mapping from $\bvars$ to $\bdomain$ and $I_{\xvars}$ is a mapping from $\xvars$ to $\rdomain$.

The logical value of an atomic formula $\theta$ under the interpretation $I$, is defined as $I(\theta) := I_{\bvars}(\theta)$ if $\theta$ is a logical proposition. In case of $\theta$ being a real arithmetical proposition, we define $I(\theta) := \true$ if the inequality $\hat{\theta}(I_{\xvars}(X_1), \ldots, I_{\xvars}(X_N)) \le 0$ holds, and $I(\theta) := \false$ otherwise. Requiring an interpretation to commute with logical connectives in the usual way extends the definition of an interpretation to any SMT formula. 
\end{definition}

The mappings $I_{\bvars}$ and $I_{\xvars}$ from the previous definition are called \emph{partial interpretations} of Boolean and continuous variables, respectively.
Analogously to the purely Boolean case, we define models of SMT formulas as $(M+N)$-tuples.

\begin{definition}[Model of an SMT formula]\label{def:model_smt}
Let $I = (I_{\bvars}, I_{\xvars})$ be an interpretation and $\phi$ an SMT formula over variables in $\bvars$ and $\xvars$ such that $I(\phi) = \true$. We say that
\begin{align*}
\mathfrak{M}(I)
&= \tuple{\mathfrak{M}_{\bvars}(I), \mathfrak{M}_{\xvars}(I)} \\
&= \tuple{\tuple{I_{\bvars}(B_i)}_{i = 1}^M , \tuple{I_{\xvars}(X_j)}_{j = 1}^N} \in \bdomain^M \times \rdomain^N
\end{align*}
is a \emph{model} of formula $\phi$ associated to interpretation $I$. The \emph{set of all models} is denoted again by $\modls{\phi}$.
\end{definition}

The projection of $\modls{\phi}$ to $\bdomain^M$ is denoted by
\[ \mathcal{M}_{\bvars}(\phi) = \set{b \in \bdomain^M \mid \text{there is } x \in\! \rdomain^N \text{ such that }\! \tuple{b, x} \in\! \modls{\phi}} \text, \]
and analogously by $\mathcal{M}_{\xvars}(\phi)$ its projection to $\mathbb{R}^N$.
These sets contain \emph{partial models} of a formula and are used below in the definition of weighted model integration.
Furthermore, for any $b \in \mathbb{B}^M$, set denoted by
\[ \mathcal{M}_{\xvars}(\phi)/b = \set{x \in \rdomain^N \mid \tuple{b, x} \in \modls{\phi}} \] 
consists of elements $x \in \mathcal{M}_{\xvars}(\phi)$ which extend partial model $b \in \mathcal{M}_{\bvars}(\phi)$ to a total model $\tuple{b, x} \in \modls{\phi}$.


\begin{definition}[WMI]\label{def:wmi}
Let $\bvars$ be a set of $M$ Boolean variables, $\xvars$ a set of $N$ real variables, and $\phi$ an SMT formula over $\bvars$ and $\xvars$.
Let $w \colon \bdomain^M \times \rdomain^N \to \mathbb{R}_{\ge 0}$ be a weight function of Boolean and real variables.
For any $b \in \bdomain$, a function $w_b \colon \rdomain^N \to \mathbb{R}$ is defined with $w_b(x) = w(b, x)$, for all $x \in \rdomain^N$.
Assume that for all $b \in \mathcal{M}_{\bvars}(\phi)$, the functions $w_{b}$ are Riemann integrable on the sets $\mathcal{M}_{\xvars}(\phi)/b$, respectively.
We define the \emph{weighted model integral (WMI)} of a formula $\phi$ with regards to the weight function $\weight$ by:
\begin{align}\label{eq:wmi}
     \wmi\left(\phi, \weight \mid \bvars, \xvars\right) =  \sum_{b \in \mathcal{M}_{\bvars}(\phi)} \, \int_{x \in  \mathcal{M}_{\xvars}(\phi)/b}\!\!\!\! w_b(x)\, d x_1\, d x_2 \dotsb d x_N \text.
\end{align}
\end{definition}


\section{Measure theoretic WMC and WMI}\label{sec:mt_wmc_wmi}

As a central contribution, we introduce variants of both WMC and WMI based on measure theory, introducing {\bf measures of weighted propositional logic and SMT formulas}.
We proceed to prove that they generalize classical WMC and WMI based on summation and Riemann integration.
This measure theoretic formulation of WMC and WMI yields an elegant proof of congruence of these two concepts in the case of a purely Boolean domain.

A formulation that treats Boolean and real variables on equal footing and leads to the congruence of WMC and WMI has also been presented in~\cite{zengefficient}, where the authors reduce weighted model integration to model integration~\cite{luu2014model}. However, the reduction is performed by transforming the summation over Boolean variables to a Riemann integration over real variables without relying on the more powerful and expressive Lebesgue integration.

Prior to formulating weighted model counting and integration as a measure theoretic problem, we give a brief introduction to measure theory. We provide a formal excursion on the measure theoretic concepts essential to this paper in Appendix~\ref{sec:measure_theory}.

\subsection{An Appetizer of Measure Theory}\label{sec:mt_appetizer}

Let us assume we have two real numbers $a$ and $b$. 
We would like to know how far these two numbers are apart. 
In other words, we would like to know the length $l$ of the segment $S$ delimited by $a$ and $b$. 
In Euclidean geometry, the length $l$ is simply given by $l=\lvert b-a \rvert$. Now, instead of viewing $l$ as the length of the segment $S$, we can also regard $l$ as the size of the set of points that make up $S$.

Measure theory generalizes the concepts of length, area and volume by answering the question {\em `how big is a specific set?'} This is done by systematically assigning a positive real number to a given set. A set is called measurable if such a number can actually be assigned. In Euclidean geometry, a measure of particular importance is the Lebesgue measure, which assigns the conventional Euclidean length, volume, and hypervolume to measurable subsets of the $N$-dimensional Euclidean space $\mathbb{R}^N$.  

Furthermore, measure theory does also provide the axiomatic formulation of probability theory as developed by~\citet{kolmogorov1950foundations}: 
probability theory considers measures that assign to the whole set (domain of definition) size 1, and considers measurable subsets to be events whose probability is given by the measure.  
This probability measure then corresponds to the expectation of random variables.

In the context of model counting (\#SAT), we want to determine/measure the size of the set of satisfying assignments to a propositional logic formula.

For the uninitiated reader we provide in Table~\ref{tab:glossary} a glossary of technical terms used in measure theory and give the relevant pointers to their introduction in Appendix~\ref{sec:measure_theory}.

\begin{table}[t]
\begin{center}
\caption{Glossary of technical terms used in measure theory.}
\label{tab:glossary}
\begin{tabular}{l|l}
$\sigma$-algebra (Def.~\ref{def:sigma_algebra})      & Lebesgue measure (Def.~\ref{ex:lebesgue})   \\
measurable space (Def.~\ref{def:sigma_algebra})    &      measurable function (Def~\ref{def:measurable_function})  \\
Borel $\sigma$-algebra (Def.~\ref{def:borel})     &        simple function (Def.~\ref{def:simple_function})         \\
countably additive (Def.~\ref{def:borel})  &      $\mu$-almost everywhere  (Def.~\ref{def:almost_everywhere})         \\
measure (Def.~\ref{def:measure}) &        product measure (Theo.~\ref{thm:prod_meas})         \\
counting measure (Def.~\ref{ex:counting}) &               probability space  (Def.~\ref{def:probability_space})

\end{tabular}
\end{center}
\end{table}

{~}

\subsection{Measure Theoretic WMC}\label{sec:mt_wmc}

In order to embed WMC into measure theory (using Lebesgue integration), a slight adjustment to Definition~\ref{def:wmc} is in order.
It is more convenient to define a weight function over the set $\bdomain^M$, similarly to Definition~\ref{def:wmi}, instead of over the set of literals $\mathcal{L}_{\bvars}$. 
To this end, we transform the given weight function $w^{\mathcal{L}} \colon \mathcal{L}_{\bvars} \to \mathbb{R}_{\ge 0}$ over literals to an equivalent weight function $w \colon \bdomain^M \to \mathbb{R}_{\ge 0}$ over $\bdomain^M$ as follows: for any $b = \tuple{b_1, b_2, \ldots, b_M} \in \bdomain^M$, let
\begin{align}\label{eq:weight}
    \weight(b) = \prod_{i=1}^M \, \ite \left(b_i, \weight^{\mathcal{L}}(B_i), \weight^{\mathcal{L}}(\lnot B_i)\right) \text.
\end{align}
Equation~\eqref{eq:wmc} now becomes:
\begin{align}
    \wmc \left(\phi, \weight^{\mathcal{L}} \mid \bvars\right) = \sum_{\mathfrak{M} \in \modls{\phi}}\! \weight(\mathfrak{M}) \nonumber \text.
\end{align}
Notice that this expression already looks `Lebesguean'. Indeed, we only need to specify the components of an appropriate measure space.

\begin{proposition}\label{prop:space_bool}
$\tuple{\mathbb{B}^M, \partitive(\mathbb{B}^M), \bmeas}$ is a measure space, where $\partitive(\mathbb{B}^M)$ is a partitive set of $\bdomain^M$ and $\bmeas \colon \partitive(\bdomain^M) \to \interval{0}{+\infty}$ is a counting measure.
\end{proposition}
\begin{proof}[\bf{Proof}]
Any set together with a counting measure on its partitive set defines a measure space.
\end{proof}

We are now in the position to express the weighted model count in measure theoretic terms.

\begin{definition}[Lebesgue WMC]\label{def:wmc_leb}
Let $\bvars$ be a set of $M$ Boolean variables, and $\phi$ a propositional formula over $\bvars$. 
Furthermore, let $\weight \colon \bdomain^M \to \mathbb{R}_{\ge 0}$ be a weight function. 
The \emph{Lebesgue weighted model count (\lwmc)} of the formula $\phi$ with respect to the weight $\weight$ is defined by:
\begin{align}
\lwmc \left(\phi, \weight\right) = \int_{\modls{\phi}} \weight\, d\bmeas \, \text. \nonumber
\end{align}
\end{definition}

The integral in the previous definition is well defined, because $\weight$ is obviously bounded (as the set $\bdomain^M$ is finite) and it is trivially measurable (as the whole partitive set of $\bdomain^M$ is a $\sigma$-algebra). 

\begin{theorem} \label{thm:wmc_leb}
Let $\bvars$ be a set of $M$ Boolean variables, and $\phi$ be a propositional formula over $\bvars$.
Furthermore, let $\weight^{\mathcal{L}} \colon \litb \to \mathbb{R}_{\ge 0}$ be a \emph{weight function} of Boolean literals and $w \colon \bdomain^M \to \mathbb{R}$ be constructed from $\weight^{\mathcal{L}}$ as in Equation~\eqref{eq:weight}.
Then:
\begin{equation*}\label{eq:wmc_leb}
    \lwmc\left(\phi, \weight\right) = \wmc\left(\phi, \weight^{\mathcal{L}} \mid \bvars\right) \text.
\end{equation*}
\end{theorem}

\begin{proof}[\bf{Proof}]
Since $\mathbb{B}^M$ is a finite set, $w$ is a simple function. $\ive{\cdot}$ will denote the Iverson bracket, which evaluates to $1$ if its argument is satisfied, and $0$ otherwise~\cite{iverson1962programming,knuth1992notes}.
\begin{align}
     \phantom{={}} \lwmc\left(\phi, \weight\right)
    = \int_{\modls{\phi}}\! w \, d\bmeas \nonumber 
    &= \int_{\bdomain^M} \left( w \cdot \indicator{\modls{\phi}} \right) d\bmeas \nonumber \\
    &= \int_{\bdomain^M} \left( \left(\, \sum_{b\in \bdomain^M} \weight(b) \cdot \indicator{\set{b}} \right) \cdot \indicator{\modls{\phi}} \right) d\bmeas  \nonumber\\
    &= \int_{\bdomain^M} \left(\, \sum_{b\in \bdomain^M} \weight(b) \cdot \indicator{\modls{\phi} \cap \set{b}} \right) d\bmeas  \nonumber \\
    &=\sum_{b\in \bdomain^M} \weight(b) \cdot \bmeas(\modls{\phi} \cap \set{b}) \nonumber \\
    &= \sum_{b\in \bdomain^M} \weight(b) \cdot \ive{b \in \modls{\phi}} \nonumber \\
    &= \sum_{b\in \modls{\phi}} \weight(b) = \wmc\left(\phi, \weight^{\mathcal{L}} \mid \bvars\right) \text. \nonumber \qedhere
\end{align}
\end{proof}

This proves that the newly defined Lebesgue weighted model count, based on measure theory, coincides with the classical weighted model count from Definition~\ref{def:wmc}. This result is the first step towards a measure theoretic formulation of WMI.

\subsection{Measure Theoretic WMI}\label{sec:mt_wmi}

We now turn to introducing an appropriate measure space for the hybrid domain consisting of Boolean and real variables, and proving the central result of this paper.
In the following, $\borel(\rdomain^N)$ denotes the Borel $\sigma$-algebra on $\rdomain^N$ from Definition~\ref{def:borel} and $\rmeas^N$ the Lebesgue measure on $\rdomain^N$ from Definition~\ref{ex:lebesgue}. The exponent in $\rmeas^N$ shall be omitted for simplicity, when the dimension of the real space is clear from context.

\begin{proposition}\label{prop:space_real}
$\tuple{\bdomain^M \times \rdomain^N, \partitive(\bdomain^M) \times \borel(\rdomain^N), \bmeas \times \rmeas}$ is a measure space, which is a product of the measure space $\tuple{\bdomain^M, \partitive(\bdomain^M), \bmeas}$ from Proposition~\ref{prop:space_bool} with measure space $\tuple{\rdomain^N, \borel(\rdomain^N), \rmeas}$.
\end{proposition}
\begin{proof}[\bf{Proof}]
See Theorem~\ref{thm:prod_meas}.
\end{proof}

Next we introduce the technical concept of \emph{measurability} of an SMT theory. 
Say that an SMT formula $\phi$ is \emph{measurable} if its set of models $\modls{\phi}$ is a mesurable set in the measure space from proposition~\ref{prop:space_real}.
Now an SMT theory is said to be \emph{measurable} if all its formulas are measurable.

\begin{lemma}
SMT(\lra), SMT(\nra) and SMT(\ra) are measurable theories.
\end{lemma}

\begin{proof}[\bf{Proof}]
Linear functions, polynomials and generally all real functions obtained by means of addition, multiplication and exponentiation of real variables and constants, are continuous. Hence, they are Borel measurable (see Example~\ref{ex:meas_functions}). 
Now note that the set of models for any real arithmetical proposition $\theta$ is $\hat{\theta}^{-1}(\interval{-\infty}{0})$. 
Therefore, these sets are Lebesgue measurable.
The set of models of logical proposition is always measurable, since $\sigma$-algebra on $\bdomain^M$ is the whole partitive set $\partitive(\bdomain^M)$.

The set of models of any formula from the above theories is now obtained as a (possibly complement of) finite union and intersection of products of models for the Boolean and real parts of the formula. By definition they remain elements of the product $\sigma$-algebra, i.e.\ they are measurable. \qedhere
\end{proof}

We have set the stage for the definition of the measure theoretic weighted model integral.

\begin{definition}[Lebesgue WMI]\label{def:wmi_leb}
Let $\bvars$ be a set of $M$ Boolean variables, $\xvars$ a set of $N$ real variables and $\phi$ an SMT formula over $\bvars$ and $\xvars$. 
Furthermore, let $w \colon \bdomain^M \times \rdomain^N \to \mathbb{R}_{\ge 0}$ be a weight function of Boolean and real variables. 
Assume that the formula $\phi$ is measurable and the function $\weight$ is integrable with regards to the product measure $\bmeas \times \rmeas$ on $\partitive(\bdomain^M) \times \borel(\rdomain^N)$ from Proposition~\ref{prop:space_real}. 
We define the \emph{Lebesgue weighted model integral (\lwmi)} of the formula $\phi$ with respect to the weight $\weight$ as:
\begin{align}
    \lwmi \left( \phi, \weight \right) = \int_{\modls{\phi}}\! w\, d(\bmeas \times \rmeas) \text. \nonumber
\end{align}
\end{definition}

For weight functions that are not Riemann integrable but Lebesgue, Definition~\ref{def:wmi_leb} provides an alternative to Defintion~\ref{def:wmi} for the weighted model intergral.
On the other hand, in case of Riemann integrable weight functions, \wmi and \lwmi are equal.


\begin{theorem}\label{thm:wmi_leb}
Under the assumptions of Definition~\ref{def:wmi}, the following equality holds:
\begin{equation*}\label{eq:wmi_leb}
    \lwmi \left( \phi, \weight \right) = \wmi \left(\phi, \weight \mid \bvars, \xvars \right) \,\text.
\end{equation*}
\end{theorem}
\begin{proof}[\bf{Proof}]
For each $b \in \bdomain^M$, functions $\weight_b$ are by assumption Riemann integrable over sets $\mathcal{M}_{\xvars}(\phi)/b$, respectively. 
This implies that the sets $\mathcal{M}_{\xvars}(\phi)/b$ are Borel measurable and that the functions $\weight_b$ are Lebesgue integrable over these sets, respectively. 
Furthermore, the set $\bdomain^M$ is finite, and the following identities clearly hold:
\begin{align*}
    & \mathcal{M}(\phi) = \bigcup\nolimits_{\beta \in \bdomain^M} \set{\beta} \times \mathcal{M}_{\xvars}(\phi)/\beta \text, \\
    & \weight(b, x) = \sum\nolimits_{\beta \in \bdomain^M} \ive{\beta = b} \cdot w_b(x) \text.
\end{align*}
We conclude that the formula $\phi$ is measurable and that the function $w$ is integrable with regards to the product measure $\bmeas \times \rmeas$. 
Finally, we obtain:

\begingroup
\allowdisplaybreaks
\begin{align}
    {\phantom{{}={}}} \lwmi \left( \phi, \weight \right) 
    =  \int_{\modls{\phi}}\! w\, d(\bmeas \times \rmeas) \nonumber 
    &= \int_{\bdomain^M \times \rdomain^N}\! \left( w \cdot \indicator{\modls{\phi}} \right) d(\bmeas \times \rmeas) \nonumber \\
    &= \int_{\bdomain^M} \left(\ \int_{\rdomain^N}\! \left( w \cdot \indicator{\modls{\phi}} \right)_{b} d\rmeas \right) d\bmeas  \tag{\text{Theorem~\ref{thm:Tonelli}}} \nonumber  \\
    &= \int_{\bdomain^M} \left(\ \int_{\rdomain^N}  w_b \cdot \ive{b \in \mathcal{M}_{\bvars}(\phi)} \cdot \indicator{\mathcal{M}_{\xvars}(\phi)/b}\, d\rmeas \right) d\bmeas \nonumber \\
    &= \int_{\bdomain^M} \left(\ \int_{\rdomain^N} w_b \cdot \indicator{\mathcal{M}_{\xvars}(\phi)/b}\, d\rmeas \right) \cdot \indicator{\mathcal{M}_{\bvars}(\phi)}\, d\bmeas \nonumber \\
    &= \sum_{b \in \mathcal{M}_{\bvars}(\phi)}  \int_{\, \mathcal{M}_{\xvars}(\phi)/b}\! w_b\, d\rmeas \nonumber \\
    &=\wmi \left( \phi, \weight \mid \bvars, \xvars \right) \text. \nonumber \tag*{\qedhere}
\end{align}
\endgroup
\end{proof}


Both, Theorem~\ref{thm:wmc_leb} and Theorem~\ref{thm:wmi_leb}, state that the weighted model count/integral of a formula is equal to the Lebesgue integral of the weight function over the set of models of a formula. 
This unification enables us to elegantly prove that \wmc is a special case of \wmi.

\begin{corollary}\label{cor:wmc_is_wmi}
Let $\bvars$ be a set of $M$ Boolean variables, and $\phi$ be a propositional formula over $\bvars$. Furthermore, let $\weight^{\mathcal{L}} \colon \litb \to \mathbb{R}_{\ge 0}$ be a weight function of Boolean literals and $w \colon \mathbb{B}^M \to \mathbb{R}_{\ge 0}$ be constructed from $\weight^{\mathcal{L}}$ as in Equation~\eqref{eq:weight}. Then:
\begin{equation}\label{eq:wmc_is_wmi}
    \wmc \left(\phi, \weight^{\mathcal{L}} \mid \bvars \right) = \wmi \left(\phi, \weight \mid \bvars, \emptyset \right)
\end{equation}
\end{corollary}
\begin{proof}[\bf{Proof}]
From the point of view of weighted model integration, this presents a degenerate case with no real variables ($\xvars = \emptyset$).
The space reduces to the Boolean space $\bdomain^M$ and the measure reduces to the Boolean measure $\bmeas$, i.e.\ $\bdomain^M \times \rdomain^0 = \bdomain^M$ and $\bmeas \times \rmeas^0 = \bmeas$.
Plugging this into Theorem~\ref{thm:wmi_leb}, together with Theorem~\ref{thm:wmc_leb}, yields
\begin{align}
\wmi(\phi, \weight \mid \bvars, \emptyset) =\! \int_{\modls{\phi}}\! w\, d\bmeas = \wmc(\phi, \weight^{\mathcal{L}} \mid \bvars) \nonumber \text. & \qedhere    
\end{align}

\end{proof}

\lwmi can now easily be extended to domains including integer variables, besides Boolean and real ones. 
The construction is completely analogous to the one presented in this section.
The  appropriate measure space is
\[ \tuple{\bdomain^M \times \rdomain^N \times \zdomain^K,\ \partitive(\bdomain^M) \times \borel(\rdomain^N) \times \partitive(\zdomain^K),\ \bmeas\times\rmeas\times\zmeas}\text, \]
where $\zmeas$ is the counting measure on $\partitive(\zdomain^K)$, and analogous results as in Theorem~\ref{thm:wmi_leb} and Corollary~\ref{cor:wmc_is_wmi} hold.


\section{Weight Functions as Measures}\label{sec:measures}

A major application of \wmc and \wmi is to be found inside probabilistic inference tasks. 
There a weight function can be regarded as a probability density function (PDF).
This enables us to define a measure (i.e.\ a probability) on the underlying space directly from a weight function. 
In this section we consider the weighted model count/integral of a logical formula in this probabilistic setting\footnote{The discussion of WMI, in this paper, is limited to finitely many variables. In probabilistic logic programming this is also called the {\em finite support condition}~\citep{sato1995statistical}. 
A possible avenue for future research is an extension to infinitely (including uncountably) many variables in the special case of WMI with probability measures, cf.~\cite{sato1995statistical}, ~\cite[Theorem 6.18]{kallenberg2002foundations}, and ~\cite{wu2018discrete}.}.
Under these assumptions, WMC and WMI equal simply the probability of the set of models.
The integration process gets encapsulated into the construction of the probability (cf.\ celebrated Radon-Nikodym derivative \citep[Theorem 6.2.3]{cohn2013measure}).
This approach can be extended beyond probabilistic measures, that is, to any finite measure which represents a weight function. It is also suitable for hybrid domains with integers, in the manner explained at the end of the previous section.

\subsection{Weighted Model Counting as Measure}\label{sec:wmc_measure}

Let again $\bvars = \set{B_1, B_2, \ldots, B_M}$ be a set of $M$ Boolean variables which form the basis of propositional logic.
Assume that weight function $w \colon \bdomain^M \to \interval{0}{1}$ is a PDF on $\bdomain^M$, i.e.\ $\sum_{b \in \bdomain^M} w(b) = 1$ holds.
The next proposition introduces a natural probability which arises from the weight function $\weight$. We refer to it as a \emph{probability associated to the weight function} $\weight$. As before, $\mu$ denotes the counting measure on $\partitive(\bdomain^M)$.

\begin{proposition}\label{prop:prob_meas_wmc}
Let $\weight \colon \bdomain^M \to \interval{0}{1}$ be a weight function such that $\sum_{b \in \bdomain^M} w(b) = 1$.
For any $B \subset \bdomain^M$, let $\bprob \colon \partitive(\bdomain^M) \to \interval{0}{1}$ be given with
\begin{align}
\bprob(B) = \sum\nolimits_{b \in B} w(b) = \int\nolimits_{B} w\, d\mu \text. \nonumber
\end{align}
Then $\tuple{\bdomain^M, \partitive(\bdomain^M), \bprob}$ is a probability space.
\end{proposition}
\begin{proof}[\bf{Proof}]
Follows trivially from the definition of $\bprob$ and the properties of $\weight$.
\end{proof}

As in Section~\ref{sec:wmc}, we describe how a probabilistic weight function of Boolean literals can naturally be transformed into a PDF on $\bdomain^M$.
Let $\weight^{\mathcal{L}} \colon \litb \to \interval{0}{1}$ be a function that, for every $i = 1$, $2$, $\dotsc$, $M$, satisfies
\[ \weight^{\mathcal{L}}\left(B_i\right) + \weight^{\mathcal{L}}\left(\lnot B_i\right) = 1 \text. \]
Using the same construction as in Equation~\eqref{eq:weight}, we get a function $\weight \colon \bdomain^M \to \interval{0}{1}$ satisfying
\begin{align}
\sum_{b \in \bdomain^M} w(b)
= \sum_{b \in \bdomain^M} \prod_{i=1}^M \,\ite\left(b_i, \weight^{\mathcal{L}}\left(B_i\right), \weight^{\mathcal{L}}\left(\lnot B_i\right)\right) 
= \prod_{i=1}^M \left( \weight^{\mathcal{L}}\left(B_i\right) + \weight^{\mathcal{L}}\left(\lnot B_i\right) \right) = 1 \text.
\end{align}

Proposition~\ref{prop:prob_meas_wmc} now produces a probability $\bprob$ associated with $\weight^{\mathcal{L}}$.
Factorization over literals indicates their independence with regards to the probability $\bprob$. 

The weighted model count is now obtained by simply measuring the size of the set of models, using the just introduced probability.

\begin{theorem}\label{thm:wmc_prob}
Let $\bvars$ be a set of $M$ Boolean variables, and $\phi$ be a propositional formula over $\bvars$. 
Let $\weight^{\mathcal{L}} \colon \litb \to \interval{0}{1}$ be a probabilistic weight function of Boolean literals and $w \colon \mathbb{B}^M \to \interval{0}{1}$ be a PDF constructed from $\weight^{\mathcal{L}}$ as in Equation~\eqref{eq:weight}. 
Furthermore, let $\bprob$ be the probability associated to the weight function $\weight$.
Then:
\begin{equation}\label{eq:wmc_prob}
    \lwmc \left(\phi, \weight^{\mathcal{L}} \mid \bvars\right) = \bprob\left(\modls{\phi}\right) \text.
\end{equation}
\end{theorem}

\begin{proof}[\bf{Proof}]
Follows directly from Definition~\ref{def:wmc_leb} and Proposition~\ref{prop:prob_meas_wmc}.
\end{proof}

\subsection{Weighted Model Integration as Measure}\label{sec:wmi_measure}

We extend the discussion from the last section to the hybrid domain. Let again $\xvars = \linebreak \set{X_1, X_2, \dotsc, X_N}$ be the set of real variables.
Given a PDF $\weight \colon \bdomain^M \times \rdomain^N \to \mathbb{R}_{\ge 0}$, we obtain the probability $\nu \colon \partitive(\bdomain^M) \times \borel(\rdomain^N) \to \interval{0}{1}$ defined with
\begin{align}\label{eq:prob_hybrid}
    \nu(E) = \int_{E} \weight \,d(\bmeas \times \rmeas) \text,
\end{align}
for every $E \in \partitive(\bdomain^M) \times \borel(\rdomain^N)$.
Because of the form of this measure, a hybrid-domain analogue of Theorem~\ref{thm:wmc_prob} is obtained effortlessly.

\begin{theorem}\label{thm:wmi_prob}
Let $\bvars$ be a set of $M$ Boolean variables, $\xvars$ a set of $N$ real variables and $\phi$ a measurable SMT formula over $\bvars$ and $\xvars$.
Furthermore, let $\weight \colon \bdomain^M \times \rdomain^N \to \mathbb{R}_{\ge 0}$ be a weight function of Boolean and real variables. 
If $\weight$ is a PDF on $\bdomain^M \times \rdomain^N$ defining probability $\nu$ given by~\eqref{eq:prob_hybrid},
then:
\begin{equation*}
    \lwmi \left(\phi, \weight \right) = \nu \left(\modls{\phi}\right) \text.
\end{equation*}
\end{theorem}

\begin{proof}[\bf{Proof}]
Follows directly from Definition~\ref{def:wmi_leb}.
\end{proof}

In the following we are concerned with the factorization of a weight function into separate parts over Boolean and continuous spaces, respectively. 
This discussion is of interest, as probabilities on Boolean and continuous spaces can be combined together using the product measure construction. 
We begin with general setting, and later comment on an important special case, where the weight function fully factorizes.

Any weight function $\weight \colon \bdomain^M \times \rdomain^N \to \mathbb{R}_{\ge 0}$ can be partially factorized such that the equality
\begin{equation}\label{eq:weight_factor}
    \weight(b, x) = \weight_{\bvars}(b) \cdot \weight^{\,b}_{\xvars}(x)
\end{equation}
holds for all $b \in \bdomain^M$ and $x \in \rdomain^N$ (note the dependency of the second factor on $b$).
The function $\weight_{b} \colon \bdomain^M \to \mathbb{R}_{\ge 0}$ is the Boolean part of the function $\weight$ and, for each $b \in \bdomain^M$, the function $\weight^{\, b}_{\xvars} \colon \rdomain^N \to \mathbb{R}_{\ge 0}$ is a piece of the continuous part.
This factorization is generally not unique. For instance, every non-zero $c \in \mathbb{R}$ defines a simple factorization, given with  $\weight_{\bvars}(b) = c$ and $\weight^{\,b}_{\xvars}(x) = \frac{1}{c} \weight(b, x)$, for all $b \in \bdomain^M$ and $x \in \rdomain^N$.
However, in the probabilistic setting, partial factorization is essentially unique:
\begin{lemma}\label{lem:unique_factor}
Let $\weight$ be a PDF on $\bdomain^M \times \rdomain^N$. Then there are \textbf{unique PDFs} $\weight_{\bvars}$ on $\bdomain^M$ and $\weight^{\, b}_{\xvars}$ on $\rdomain^N$, for each $b \in \bdomain^M$, such that Equality~\eqref{eq:weight_factor} holds for every $b \in \bdomain^M$ and for $\rmeas$-almost every $x \in \rdomain^N$ (cf. Definition~\ref{def:almost_everywhere}).
\end{lemma}
\begin{proof}[\bf{Proof}]
For each $b \in \bdomain^M$, denote again with $\weight_b \colon \rdomain^N \to \mathbb{R}_{\ge 0}$ a function given with $\weight_b(x) = \weight(b, x)$.
Define the function $\weight_{\bvars} \colon \bdomain^M \to \mathbb{R}_{\ge 0}$ with
$\weight_{\bvars}(b) = \int_{\rdomain^N} \weight_b \, d\rmeas \text{ , for every } b \in \bdomain^M \text. $
Now for each $b \in \bdomain^M$, define the functions $\weight^{\,b}_{\xvars} \colon \rdomain^N \to \mathbb{R}_{\ge 0}$ with $\weight^{\,b}_{\xvars}(x) = \frac{\weight(b, x)}{\weight_{\bvars}(b)}$ if $\weight_{\bvars}(b) \ne 0$, and $\weight^{\,b}_{\xvars}(x) = 0$ otherwise, for every $x \in \rdomain^N$.
In the former case, Equation~\eqref{eq:weight_factor} obviously holds. In the latter case, from Lemma~\ref{lem:int_of_nonneg_func} we conclude that $\weight_b(x) = \weight(b, x) = 0$ for $\rmeas$-almost every $x \in \rdomain^N$, and therefore Equation~\eqref{eq:weight_factor} holds $\rmeas$-almost everywhere on $\rdomain^N$.

Using the fact that $\weight$ is a PDF together with Theorem~\ref{thm:Tonelli}, we prove that $\weight_{\bvars}$ is a PDF as well:
\begin{align}
\int_{\bdomain^M} \weight_{\bvars} \, d\bmeas
= \int_{\bdomain^M} \left( \int_{\rdomain^N} \weight(b, x) \, d\rmeas \right) d\bmeas
= \int_{\bdomain^M \times \rdomain^N} \weight \,d\left(\bmeas \times \rmeas\right) = 1 \text.
\end{align}

Proving that $\weight^{\,b}_{\xvars}$ is a PDF, for each $b \in \bdomain^M$, is trivial. From
\begin{align}
1 = \int_{\rdomain^N} \weight^{\,b}_{\xvars} \, d\rmeas = \int_{\rdomain^N} \frac{\weight(b, x)}{\weight_{\bvars}(b)} \, d\rmeas = \frac{1}{\weight_{\bvars}(b)} \int_{\rdomain^N} \weight_b \, d\rmeas \nonumber 
\end{align}

it follows that $\weight_{\bvars}$ is unique, and then $\weight^{\,b}_{\xvars}$ as well.
\end{proof}

As a consequence, we can split the probability $\nu$ from Equation~\eqref{eq:prob_hybrid} into Boolean and continuous parts.
For a given PDF $\weight$ on $\bdomain \times \rdomain^N$, we first find its unique factors from Lemma~\ref{lem:unique_factor}, that is the PDFs $\weight_{\bvars}$ on $\bdomain^M$ and $\weight^{\, b}_{\xvars}$ on $\rdomain^N$, for each $b \in \bdomain^M$.
Each function $\weight^{\, b}_{\xvars}$ defines a probability $\rprob^{\,b}$ on $\borel(\rdomain^N)$ given with
\begin{align}\label{eq:rprob}
    \rprob^{\,b}(E) = \int_{E} \weight_{\xvars}^{\,b} \, d\rmeas \text,
\end{align}
for every set $E \in \borel(\rdomain^N)$.
Using a construction similar to that of the product measure in Theorem~\ref{thm:prod_meas}, the probability $\bprob$ associated to the weight function $\weight_{\bvars}$ (from Proposition~\ref{prop:prob_meas_wmc}) can be joined with the probabilities $\rprob^{\,b}$, $b \in \bdomain^M$, in order to obtain a single probability on $\bdomain^M \times \mathbb{R}^N$.

\begin{proposition}\label{prop:space_prob_wmi}
Let $\bprob$ be a probability on $\partitive(\bdomain^M)$. 
For every $b \in \bdomain^M$, let $\rprob^{\,b}$ be a probability on $\borel(\rdomain^N)$.
Furthermore, let $\bprob \times \rprob \colon \partitive(\bdomain^M) \times \borel(\rdomain^N) \to \interval{0}{1}$ denote a function given with
\[ (\bprob \times \rprob)\left(E\right) = \int\nolimits_{\bdomain^M} \rprob^{\,b} \left(E_b\right)\, d\bprob \]
for any set $E \in \partitive(\bdomain^M) \times \borel(\rdomain^N)$, with notation the $E_b$ being explained in Section~\ref{sec:measure_theory}. The tuple 
$$\tuple{\bdomain^M \times \rdomain^N, \partitive(\bdomain^M) \times \borel(\rdomain^N), \bprob \times \rprob}$$
is a probability space.
\end{proposition}

\begin{proof}[\bf{Proof}]
The function $\bprob \times \rprob$ is countably additive, analogous to the proof of Theorem~\ref{thm:prod_meas}. The equality
$$(\bprob \times \rprob)\left(\bdomain^M \times \rdomain^N\right) {=} 1$$
follows immediately, since $\bprob$ and $\rprob^{\,b}$, for $b \in \bdomain^M$, are all probabilities.
\end{proof}

The measure $\bprob \times \rprob$ is not a product measure, since $\rprob$ alone has no meaning, yet. 
Below we describe an aforementioned important case when a weight function is fully factorized.
This measure is indeed a product measure, offering a motivation for this notation.
But first, let us rephrase the statement of Theorem~\ref{thm:wmi_prob}, in accordance with our present discussion.

\begin{corollary}\label{cor:lwmi_factor}
Let $\bvars$ be a set of $M$ Boolean variables, $\xvars$ a set of $N$ real variables, and $\phi$ a measurable SMT formula over variables in $\bvars$ and $\xvars$. 
Let $\weight \colon \bdomain^M \times \rdomain^N \to \mathbb{R}_{\ge 0}$ be a weight function of Boolean and real variables, which is a PDF on $\bdomain^M \times \rdomain^N$. 
Now let $w_{\bvars} \colon \bdomain^M \to \mathbb{R}_{\ge 0}$ and, for each $b \in \bdomain^M$, $w^{\,b}_{\xvars} \colon \bdomain^M \to \mathbb{R}_{\ge 0}$ be unique PDFs such that $\weight(b, x) = \weight_{\bvars}(b) \cdot \weight^{\,b}_{\xvars}(x)$ holds for every $b \in \bdomain^M$ and $\rmeas$-almost every $x \in \rdomain^M$. 
Furthermore, let $\bprob$ be a probability on $\partitive(\bdomain^M)$ associated to the PDF $\weight_{\bvars}$ and, for each $b \in \bdomain^M$, let $\rprob^{\,b}$ be a probability on $\borel(\rdomain^N)$ associated to the PDF $w^{\,b}_{\xvars}$.
Lastly, let $\bprob \times \rprob$ be a probability measure on $\partitive(\bdomain^M) \times \borel(\rdomain^N)$ associated to the PDF $\weight$ obtained from the probabilities $\bprob$ and $\rprob^{\,b}$, for $b \in \bdomain^M$, using Proposition~\ref{prop:space_prob_wmi}.
Then:
\begin{equation*}
    \lwmi \left(\phi, \weight\right) = \left(\bprob \times \rprob\right)\left(\modls{\phi}\right) \text.
\end{equation*}
\end{corollary}

\begin{proof}[\bf{Proof}]
We prove that $\left(\bprob \times \rprob\right)$ equals the probability $\nu$ from Equation~\eqref{eq:prob_hybrid}.
Then the claim follows by Theorem~\ref{thm:wmi_prob}.
We note that $\bprob(\set{b}) = w_{\bvars}\left(b\right) \cdot \bmeas(\set{b})$.
Because of the factorization of the weight function $\weight$, for every $b \in \bdomain^M$, we have $\weight_b = \weight_{\bvars}(b) \cdot \weight_{\xvars}^{\,b}$.
Now for any $E \in \partitive(\bdomain^M) \times \borel(\rdomain^N)$:
\begin{align}
    \nu(E) = \int_{E} \weight\, d\left(\bmeas \times \rmeas\right) \nonumber
    &= \int_{\bdomain^M} \left(\; \int_{\rdomain^N} \weight_b \cdot \indicator{{E}_{b}}\, d\rmeas \right) d\bmeas \nonumber \\
    &= \sum_{b \in \bdomain^M} \left( \int_{E_b} \weight_{\xvars}^{\,b}\, d\rmeas \right) \cdot \weight_{\bvars}(b) \cdot \bmeas(\set{b}) \nonumber \\
    &= \int_{\bdomain^M} \rprob^{\,b}(E_b)\, d\bprob = \left(\bprob \times \rprob\right)\left( E \right) \text. \nonumber  \qedhere
\end{align}
\end{proof}

Lastly, we discuss the announced case of the full factorization of the weight function~$\weight$.
In practice, there is commonly a single PDF $\weight_{\xvars}$ associated to any $b \in \bdomain^M$, i.e.\ equality
\[ \weight(b, x) = \weight_{\bvars}(b) \cdot \weight_{\xvars}(x) \]
holds for every $b \in \bdomain^M$ and $\rmeas$-almost every $x \in \rdomain^N$.
Consequently, there is one probability measure $\rprob$ on $\borel(\rdomain^N)$.
Uniqueness of the product measure from Theorem~\ref{thm:prod_meas} then implies that the probability space from Proposition~\ref{prop:space_prob_wmi} is actually a product of probability spaces $\left(\, \bdomain^M, \partitive(\bdomain^M), \bprob\,\right)$ and $\left(\, \rdomain^M, \borel(\rdomain^M), \rprob\,\right)$.
Result analogous to that of Corollary~\ref{cor:lwmi_factor} is valid in this case.

\section{Conclusion}\label{sec:conclusion}

WMI is an essential framework for solving probabilistic inference problems in discrete-continuous domains.
In this paper we present a  measure theoretic formulation of WMI using Lebesgue integration.
Consequently, we have ensured conditions for the uniform treatment of problems in Boolean, discrete, continuous domains, and mixtures thereof, which has always been a challenge using classical (Riemannian) theory of integration.
Moreover, we have provided clear terminology and precise notation based in measure theory for WMI, putting WMI on steady-state theoretical footing.
Although a direct application of the here-presented measure theoretic formulation of WMI to  building practical WMI solvers seems to be of a limited character, recent advances in using Lebesgue integration for solving integration and related problems demonstrate potential~\cite{malyshkin2018lebesgue}.
Furthermore, the well-behavedness of Lebesgue integration, with regards to limiting processes, can find its use inside probabilistic inference with potentially infinite number of variables.
This is in concordance with a current trend in probabilistic programming research, where an increasing number of papers discuss probabilistic programming from a measure theoretic perspective~\cite{narayanan2016probabilistic,heunen2017convenient,wu2018discrete}.

\appendix

\renewcommand{\thetheorem}{A\arabic{theorem}}

\section{Background on Measure Theory}\label{sec:measure_theory}

The theory presented here is taken from~\cite{cohn2013measure}. The reader not familiar with measure theory is encouraged to read this introduction for background and motivation, as well as the relationship to classical Riemann integration.

\begin{definition}[$\sigma$-algebra]\label{def:sigma_algebra}
Let $X$ be an arbitrary set. A collection $\mathcal{A}$ of subsets of $X$ is a $\sigma$-\emph{algebra} on $X$ if:
\begin{itemize}
    \item $X \in \mathcal{A}$,
    \item for each set $A$ that belongs to $\mathcal{A}$, its complement $A^{c}$ belongs to $\mathcal{A}$,
    \item for each infinite sequence $\set{A_i}$ of sets that belong to $\mathcal{A}$, set $\bigcup_{i=1}^{\infty} A_i$ belongs to $\mathcal{A}$
\end{itemize}
The pair $\tuple{X, \mathcal{A}}$ is referred to as \emph{measurable space}.
\end{definition}

It is easy to see that the intersection of two $\sigma$-algebras is again a $\sigma$-algebra. Hence, we can define the smallest $\sigma$-algebra which contains the given subsets; it is called the $\sigma$-algebra \emph{generated} by these subsets. Now we can introduce an important $\sigma$-algebra on the set $\rdomain^N$.

\begin{definition}[Borel $\sigma$-algebra]\label{def:borel}
The $\sigma$-algebra generated by the collection of all rectangles in $\rdomain^N$ that have the form
\[ \set{(x_1, \ldots, x_N) \mid a_i < x_i \le b_i \text{, for } i = 1, \dotsc, N}\]
is called \emph{Borel }$\sigma$\emph{-algebra} and is denoted with $\borel(\rdomain^N)$.
\end{definition}

A function $\mu$ from a $\sigma$-algebra $\mathcal{A}$ to $\interval{0}{+\infty}$ is said to be \emph{countably additive} if it satisfies
\[ \mu \left(\bigcup_{i=1}^{\infty} A_i \right) = \sum_{i=1}^{\infty} \mu(A_i)\]
for each infinite sequence $\set{A_i}$ of disjoint sets from $\mathcal{A}$.

\begin{definition} \label{def:measure}
Let $\mathcal{A}$ be a $\sigma$-algebra on the set $X$. The function $\mu \colon \mathcal{A} \to \interval{0}{+\infty}$ is a \emph{measure} on $\mathcal{A}$ if $\mu(\emptyset) = 0$ and $\mu$ is countably additive. The triple $\tuple{X, \mathcal{A}, \mu}$ is said to be a \emph{measure space}.
\end{definition}

We now introduce two important measures used in this paper.

\begin{definition}[Counting measure]\label{ex:counting}
Let $X$ be an arbitrary set, and $\partitive(X)$ the set of all subsets of $X$ (\emph{partitive} or \emph{power} set). Then $\partitive(X)$ is trivially a $\sigma$-algebra on $X$. Define a function $\mu \colon \partitive(X) \to \interval{0}{+\infty}$ by letting $\mu(A)$ be $n$ if $A$ is a finite set with exactly $n$ elements, and $+\infty$ otherwise. Then $\mu$ is a measure on $\partitive(X)$ called \emph{counting measure} on $\tuple{X, \partitive(X)}$.
\end{definition}

\begin{definition}[Lebesgue measure]\label{ex:lebesgue}
It is possible to construct a function $\lambda^N \colon \borel(\rdomain^N) \to \interval{0}{+\infty}$ which assigns to each rectangle $R = \set{(x_1, \dotsc, x_N) \mid a_i < x_i \le b_i \text{, for } i = 1, \ldots, N}$ its volume, i.e.\ $\lambda^N(R) = \prod_{i=1}^N (b_i - a_i)$. 
Then extending $\lambda^N$ to any set from $\borel(\rdomain^N)$ is accomplished by using countable additivity.
The function $\lambda^N$ is a measure on $\borel(\rdomain^N)$ and is known as the \emph{Lebesgue measure} on $\rdomain^N$.
\end{definition}

Let $\mu$ be a measure on a measurable space $(X, \mathcal{A})$. Then $\mu$ is a \emph{finite measure} if $\mu(X) < +\infty$ and is a $\sigma$-\emph{finite measure} if $X$ is the union of a sequence $A_1$, $A_2$, $\ldots$ of sets that belong to $\mathcal{A}$ and satisfy $\mu(A_i) < +\infty$ for each $i = 1$, $2$, $\dotsc$.

\begin{definition}[Measurable function]\label{def:measurable_function}
Let $(X, \mathcal{A})$ be measurable space. The function $f \colon X \to \interval{-\infty}{+\infty}$ is said to be \emph{measurable with respect to} $\mathcal{A}$ if for each real number $t$ the set \linebreak $\set{x \in X \mid f(x) \le t}$ belongs to $\mathcal{A}$. In the case of $X = \rdomain^N$, a function that is measurable with respect to $\borel(\rdomain^N)$ is called \emph{Borel measurable}.
\end{definition}

\begin{example}\label{ex:meas_functions}
There are some familiar measurable functions. For instance, any measurable set $B \in \mathcal{A}$ gives rise to a measurable function. Namely, its characteristic function $\indicator{B}$ is measurable with respect to $\mathcal{A}$, as both $B$ and its complement $B^c$ belong to $\sigma$-algebra $\mathcal{A}$.
On the other end of spectrum, any continuous function $f \colon \rdomain^N \to \mathbb{R}$ is Borel measurable, because $\tuple{\rdomain^N, \borel(\rdomain^N)}$ is a topological space.
\end{example}

\begin{definition}[Simple function]\label{def:simple_function}
 Let $(X, \mathcal{A})$ be a measurable space. Function $f \colon X \to \linebreak \interval{-\infty}{+\infty}$ is called \emph{simple function} if it has only finitely many different values.
\end{definition}

Let $\alpha_1$, $\alpha_2$, $\dotsc$, $\alpha_n$ be all distinct values of simple function $f$ on measurable space $(X, \mathcal{A})$. Then $f$ can be written as $f = \sum_{i=1}^n \alpha_i \indicator{{A_i}}$, where $A_i = \set{x \in X \mid f(x) = \alpha_i }$. Function $f$ is $\mathcal{A}$-measurable if and only if $A_i \in \mathcal{A}$ for all $i = 1$, $2$, $\dotsc$, $n$.
Simple functions are instrumental in the construction of Lebesgue integral. Their integral is easy to compute, while the next proposition shows that they can approximate any measurable function.

\begin{proposition}\label{prop:simple_func}
Let $(X, \mathcal{A})$ be a measurable space, and let $f$ be a $\interval{0}{+\infty}$-valued measurable function on $X$. Then there is a sequence $\set{f_n}$ of simple $[0, +\infty)$-valued measurable functions on $X$ that satisfy $f_1(x) \le f_2(x) \le \cdots$ and $f(x) = \lim_n f_n(x)$ at each $x \in X$.
\end{proposition}
\begin{proof}
See~\cite[Proposition $2.1.8$]{cohn2013measure}.
\end{proof}
The construction of integrals takes place in three stages. First, we define an integral of positive simple functions. Using Proposition~\ref{prop:simple_func}, the definition is then extended to any positive measurable function, and finally extended to the subset of all measurable functions. We denote with $f^{+}$ the function $f(x) = \max\set{0, f(x)}$, i.e.\ the positive part of function $f$, and analogously with $f^{-}$ the function $f(x) = \min\set{0, f(x)}$. The function $f$ can now be written as $f = f^+ - f^-$.

\begin{definition}[Integral]
Let $\mu$ be a measure on $\tuple{X, \mathcal{A}}$. If $f$ is a positive simple function given by $f = \sum_{i=1}^n \alpha_i \indicator{{A_i}}$, where $a_1$, $a_2$, \ldots, $a_n$ are nonnegative real numbers and $A_i$, $A_2$, \ldots, $A_n$ are disjoint subsets of $X$ that belong to $\mathcal{A}$, then $\int f\,d\mu$, \emph{the integral of} $f$ \emph{with respect to} $\mu$, is defined to be $\sum_{i=1}^n a_i\, \mu(A_i)$.

For an arbitrary $\interval{0}{+\infty}$-valued $\mathcal{A}$-measurable function on $X$ we define its integral as
\begin{align*}
 \int f\,d\mu = 
 \sup \set{\int g \,d\mu \mid g \text{ is a simple positive function and } g \le f} \nonumber \text.
 \end{align*}
Finally, let $f$ be any $\interval{-\infty}{+\infty}$-valued $\mathcal{A}$-measurable function on $X$. If both $\int f^{+} \,d\mu$ and $\int f^{-} \,d\mu$ are finite, then $f$ is called $\mu$-integrable and its integral is defined by
\[ \int f\,d\mu = \int f^{+} \,d\mu - \int f^{-} \,d\mu \text.\]
\end{definition}

Suppose that $f \colon X \to \interval{-\infty}{+\infty}$ is $\mathcal{A}$-measurable and that $A \in \mathcal{A}$. Then $f$ is \emph{integrable over} the set $A$ if the function $f \cdot \indicator{A}$ is integrable. In this case $\int_{A} f \,d\mu$, the integral of $f$ over $A$, is defined to be $\int f \,\indicator{A} \,d\mu$.

In the case of $X = \rdomain^N$ and $\mu = \lambda$, above integral is often referred to as the Lebesgue integral. The Lebesgue integral satisfies all usual basic properties of the Riemann integral (linearity and monotonicity). Importantly, Lebesgue integral equals the Riemman integral for any Riemman integrable function. There are, however, functions which are Lebesgue integrable, but not Riemann integrable.


\begin{definition}\label{def:almost_everywhere}
Let $\tuple{X, \mathcal{A}, \mu}$ be a measure space. We say that property $P$ holds $\mu$\emph{-almost everywhere} on $X$ or for $\mu$\emph{-almost every} $x \in X$ ($\,\mu${\emph{-a.e.}}) if there is a set $N \in \mathcal{A}$ such that $P$ holds for every $x \in X \setminus N$ and $\mu(N) = 0$. We omit the mention of measure $\mu$, when it is clear from context.
\end{definition}

\begin{lemma}\label{lem:int_of_nonneg_func}
Let $f \colon \rdomain^N \to \interval{-\infty}{+\infty}$ be Lebesgue integrable function. Then $\int \vert f \vert = 0$ if and only if $f = 0$ almost everywhere.
\end{lemma}
\begin{proof}[\bf{Proof}]
See~\cite[Corollary $2.3.12$]{cohn2013measure}.
\end{proof}

Now we turn to the construction of {\bf product measures}, which combines two measure spaces. Let $(X, \mathcal{A})$ and $(Y, \mathcal{B})$ be two measurable spaces, and let $X \times Y$ be the Cartesian product of the sets $X$ and $Y$. A subset of $X \times Y$ is a rectangle with measurable sides if it has the form $A \times B$ for some $A$ in $\mathcal{A}$ and some $B$ in $\mathcal{B}$. The $\sigma$-algebra on $X \times Y$ generated by the collection of all rectangles with measurable sides is called the product of the $\sigma$-algebras $\mathcal{A}$ and $\mathcal{B}$ and is denoted by $\mathcal{A} \times \mathcal{B}$.

Let $E$ be a subset of $X \times Y$. Then for each $x \in X$ and each $y \in Y$ the \emph{sections} $E_x$ and $E^y$ are subsets of $Y$ and $X$, respectively, given by $E_x = \set{ z \in Y \mid (x, z) \in E}$ and $E^y = \set{ z \in X \mid (z, y) \in E}$. If $f$ is a function on $X \times Y$, then the \emph{sections} $f_x$ and $f^y$ are functions on $Y$ and $X$, respectively, given by $f_x(z) = f(x, z)$ and $f^y(z) = f(z, y)$.

\begin{theorem}[Product measure]\label{thm:prod_meas}
Let $(X, \mathcal{A}, \mu)$ and $(Y, \mathcal{B}, \nu)$ be $\sigma$-finite measure spaces. Then there is a \textbf{unique measure} $\mu \times \nu$ on the $\sigma$-algebra $A \times B$ such that
\[(\mu \times \nu)(A \times B) = \mu(A) \nu(B)\]
holds for each $A \in \mathcal{A}$ and $B \in \mathcal{B}$. Furthermore, the measure under $\mu \times \nu$ of an arbitrary set $E$ in $A \times B$ is given by
\begin{talign}
 (\mu \times \nu)(E) = \int_X \nu(E_x) \,d\mu = \int_Y \mu(E^y) \,d\nu \nonumber \text.
\end{talign}
The measure $\mu \times \nu$ is called the product measure of $\mu$ and $\nu$.
\end{theorem}
\begin{proof}
See~\cite[Theorem $5.1.4$]{cohn2013measure}.
\end{proof}
Intgrals with respect to product measure can now be evaluated using Tonelli's theorem, a special case of Fubini's theorem.

\begin{theorem}[Tonelli's theorem]\label{thm:Tonelli}
Let $\tuple{X, \mathcal{A}, \mu}$ and $\tuple{Y, \mathcal{B}, \nu}$ be $\sigma$-finite measure spaces, and let $f \colon X \times Y \to \interval{0}{+\infty}$ be $\left(\mathcal{A} \times \mathcal{B}\right)$-measurable. Then
\begin{enumerate}
    \item[(a)] the function $x \mapsto \int_{Y}f_x\, d\nu$ is $\mathcal{A}$-measurable and the function $y \mapsto \int_{X}f^y\, d\mu$ is $\mathcal{B}$-measurable, and
    \item[(b)] $f$ satisfies
    \begin{align*}
        \int_{X \times Y} f\,d(\mu \times \nu) = \int_X \left( \int_Y f_x\, d\nu \right) d\mu = \int_Y \left( \int_X f^y\, d\mu \right) d\nu \text.
    \end{align*}
\end{enumerate}
\end{theorem}
\begin{proof}
See~\cite[Theorem $5.2.1$]{cohn2013measure}.
\end{proof}

{\bf Probability theory} is naturally expressed in terms of Lebesgue integration.

\begin{definition}\label{def:probability_space}
A \emph{probability space} is a measure space $(\Omega, \mathcal{A}, \mathbb{P})$ such that $\mathbb{P}(\Omega) = 1$. A measure $\mathbb{P}$ is called a \emph{probability}.
\end{definition}

Let now $\tuple{X, \mathcal{A}, \mu}$ be a measure space.
Suppose that $f$ is a nonnegative $\mu$-measurable function on $X$ such that $\int\nolimits_{X} f\, d\mu = 1$. Then the function $\mathbb{P} \colon \mathcal{A} \to \interval{0}{1}$ given with $\mathbb{P}(A) = \int_A f\, d\lambda$, for every $A \in \mathcal{A}$, defines a probability on the measurable space $\tuple{X, \mathcal{A}}$. The function $f$ is called the \emph{probability density function (PDF)} of probability $\mathbb{P}$.

\newpage

 \bibliographystyle{elsarticle-num-names} 
 \bibliography{references}

\end{document}